\documentclass{article}

\usepackage[utf8]{inputenc}
\usepackage[T1]{fontenc}
\usepackage{amsmath,amssymb,amsthm}
\usepackage{mathtools}
\usepackage{graphicx}
\usepackage{booktabs}
\usepackage{xcolor}
\usepackage{microtype}
\usepackage{hyperref}
\usepackage{subcaption}
\usepackage{enumitem}
\usepackage{caption}
\usepackage{algorithm}
\usepackage{algorithmic}
\usepackage{natbib}

\newtheorem{theorem}{Theorem}[section]
\newtheorem{lemma}[theorem]{Lemma}
\newtheorem{assumption}[theorem]{Assumption}
\newtheorem{definition}[theorem]{Definition}
\newtheorem{remark}[theorem]{Remark}
\newtheorem{corollary}[theorem]{Corollary}
\newtheorem{proposition}[theorem]{Proposition}

\hypersetup{
    colorlinks=true,
    linkcolor=blue,
    citecolor=blue,
    urlcolor=blue
}

\title{
Avoiding Premature Collapse: \\ Adaptive Annealing for Entropy-Regularized Structural Inference
}

\author{Yizhi Liu \\ \small Department of Computer Science, Stony Brook University \\ \small \texttt{liuyizhi774@gmail.com}}
\date{January 23, 2026}

\begin{document}
\maketitle

\begin{abstract}
Differentiable matching layers and residual connection paradigms, often implemented via entropy-regularized Optimal Transport (OT) \cite{cuturi2013sinkhorn, peyre2019computational}, serve as critical mechanisms in structural prediction and architectural scaling. However, recovering discrete permutations or maintaining identity mappings via annealing $\epsilon\rightarrow0$ is notoriously unstable \cite{peyre2019computational}. We identify a fundamental mechanism for this failure: \textbf{Premature Mode Collapse}. By analyzing the non-normal dynamics of the Sinkhorn fixed-point map \cite{trefethen2005spectra}, we reveal a theoretical thermodynamic speed limit: standard exponential cooling outpaces the contraction rate of the inference operator, which degrades as $O(1/\epsilon)$ \cite{cominetti1994asymptotic, weed2018explicit}. To address this, we propose \textbf{Efficient Piecewise Hybrid Adaptive Stability Control (EPH-ASC)}, an adaptive scheduling algorithm that monitors the stability of the inference process. We demonstrate that EPH-ASC is essential for stabilizing \textbf{Manifold-Constrained Hyper-Connections (mHC)} \cite{deepseek2025mhc} during large-scale training on the FineWeb-Edu dataset, preventing late-stage gradient explosions by enforcing a linear stability law.
\end{abstract}

\section{Introduction}
Entropy-regularized Optimal Transport (OT) has become a standard surrogate for combinatorial inference \cite{cuturi2013sinkhorn, altschuler2017near} and modern macro-architecture design \cite{deepseek2025mhc, zhu2024hyper}. Practitioners often attempt to recover hard assignments or specialized routing by annealing the regularization parameter $\epsilon\rightarrow0$ \cite{genevay2018learning, peyre2019computational}. Recently, studies such as \textbf{Manifold-Constrained Hyper-Connections (mHC)} \cite{deepseek2025mhc} have utilized the Sinkhorn-Knopp algorithm \cite{sinkhorn1967concerning} to project residual mappings onto the Birkhoff polytope, thereby restoring the identity mapping property in multi-stream architectures \cite{deepseek2025mhc}.

Empirically, however, this cooling process is fragile \cite{peyre2019computational, blondel2018smooth}. As $\epsilon$ decreases, the sensitivity of the optimal plan to cost perturbations blows up as $O(1/\epsilon)$ \cite{cominetti1994asymptotic, weed2018explicit}, a phenomenon we define as the \textbf{Thermodynamic Speed Limit}. In the presence of high-variance gradients from real-world datasets like FineWeb-Edu, standard exponential annealing \cite{chizat2024annealing} violates this limit, causing the composite mapping in HC architectures to diverge or collapse \cite{deepseek2025mhc}. We propose EPH-ASC to reconcile this instability by monitoring the \textbf{Primal Drift} \cite{weed2018explicit, eisenberger2022implicit} and triggering a "Thermodynamic Pause" when the distributional shift exceeds the solver's restoring capacity.

\section{The Mechanism of Inference Collapse}
\label{sec:dynamics}

\subsection{Geometric Picture: Basin Shrinkage}
As the temperature $\varepsilon \to 0$, the entropic map $\mathcal S_\varepsilon(C)$ sharpens soft belief states into near-permutations. Geometrically, the transport polytope decomposes into basins of attraction around permutation vertices. The failure mode we study—\emph{early locking}—occurs when the current inference state $P_t$ is attracted into a wrong basin because the optimal posterior drifted faster than the fixed-point iteration could correct. Figure~\ref{fig:trap_intro} visualizes this premature collapse.

\begin{figure}[t!]
\centering
\includegraphics[width=0.85\columnwidth]{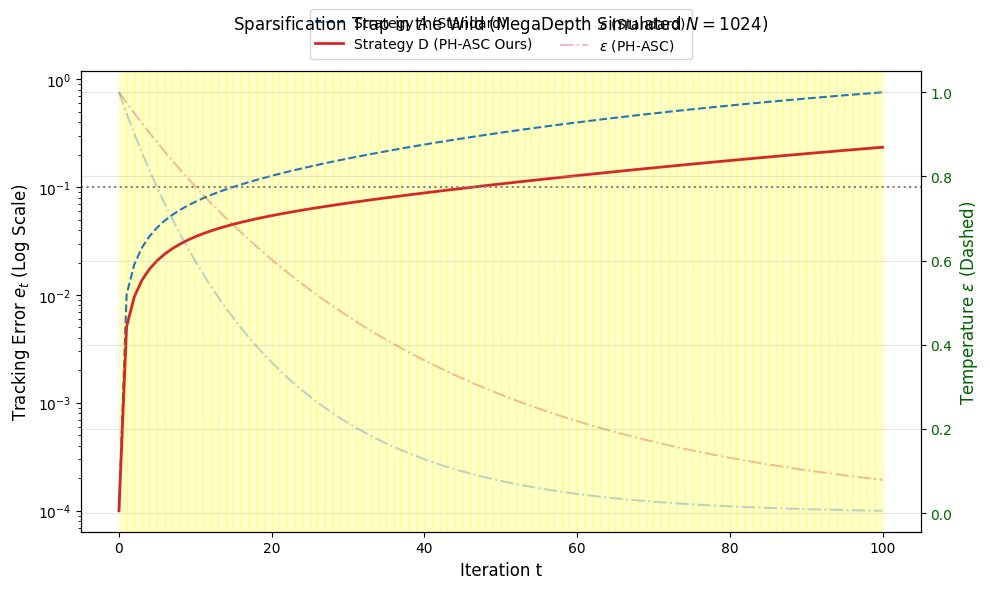}
\caption{\textbf{Premature Mode Collapse.} Standard annealing (blue) breaches the stability threshold $R$ (dotted), causing early locking into a spurious mode. Ours (red) detects the stability violation and pauses cooling. (Simulation)}
\label{fig:trap_intro}
\end{figure}

\subsection{Thermodynamic Sensitivity \texorpdfstring{($1/\varepsilon$)}{(1/epsilon)}}
To quantify the "velocity" of the distribution shift, we adopt a localized non-degeneracy assumption that fixes an active support set $S$ for small $\varepsilon$. Under this assumption, the sensitivity of the Sinkhorn map with respect to $\varepsilon$ scales like $1/\varepsilon$. This follows from implicit differentiation of the entropic optimality conditions \cite{cominetti1994asymptotic, weed2018explicit}.

\textbf{Remark (Linear Stability Scaling).} Since the restoring force of the inference operator collapses and sensitivity scales as $O(1/\varepsilon)$, the radius of the effective stability basin shrinks proportionally with $\varepsilon$. This suggests that the permissible distributional shift $\tau_t$ must follow a linear scaling law $\tau_t \propto \varepsilon$, a property we exploit in Section \ref{sec:method} to design an efficient adaptive schedule.

\subsection{Transient Inference Error and Pseudospectra}
Linearizing the Sinkhorn fixed-point map yields a Jacobian $J_\varepsilon$. While its eigenvalues describe asymptotic contraction, its \textbf{non-normal structure} creates a "shear" effect that can amplify transient inference errors.
Pseudospectral theory quantifies how contours of the resolvent extend beyond the spectral radius \cite{trefethen2005spectra}. We show (Theorem A.2) that the modal condition number $\kappa(V)$ of the Jacobian effectively compresses the basin of attraction by roughly $\kappa(V)$.

\begin{proposition}[Linear Scaling of the Stability Basin]
\label{prop:linear_scaling}
Let the cost matrix $C$ satisfy the localized non-degeneracy Assumption (stable active support $S$). For sufficiently small $\epsilon$, the spectral gap of the Sinkhorn Jacobian $J_\epsilon$ satisfies $1 - \rho(J_\epsilon) \ge \gamma \cdot \epsilon$. Consequently, to ensure the inference error $e_t$ remains within a local linearization region $R(\epsilon)$ (whose size may depend on $\epsilon$ but vanishes no faster than linearly as $\epsilon \to 0$), the permissible drift $\tau_{max}$ must scale linearly:
\begin{equation}
    \tau_{max}(\epsilon) \le \frac{R(\epsilon) \cdot \gamma}{\kappa(V)} \cdot \epsilon
\end{equation}
\end{proposition}

\begin{figure}[t!]
    \centering
    \begin{subfigure}[b]{0.48\linewidth}
        \centering
        \includegraphics[width=\linewidth]{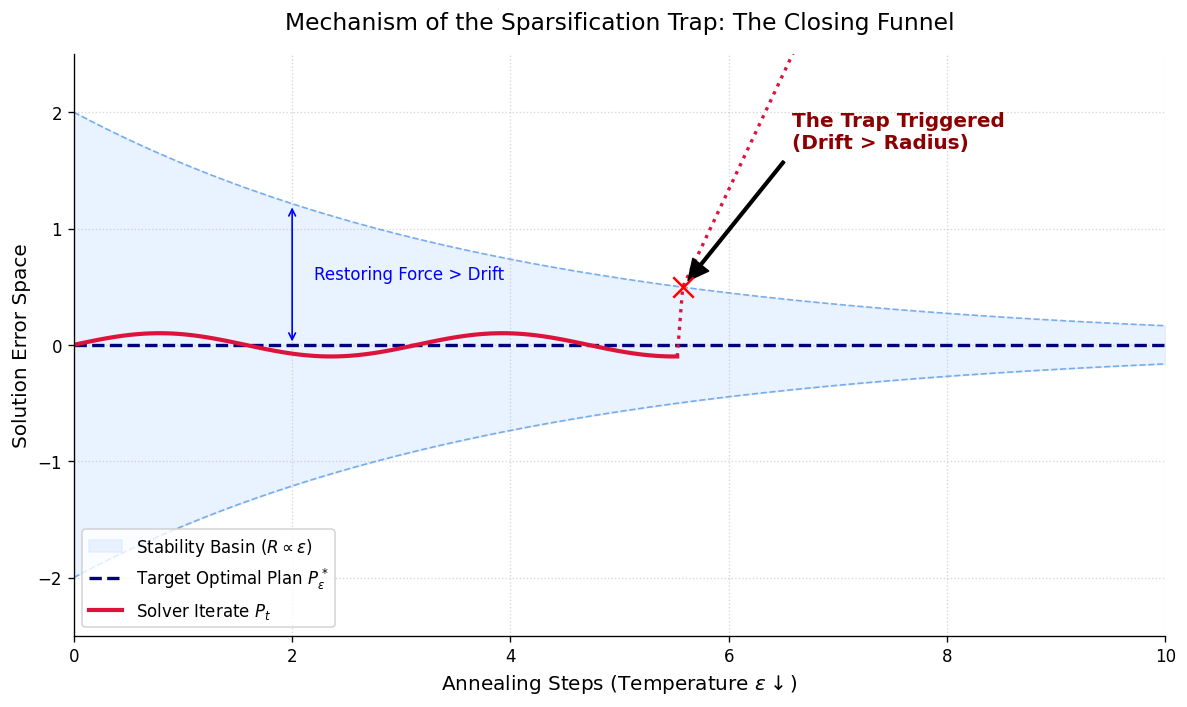} 
        \caption{\textbf{Macroscopic: The Closing Funnel.} 
        As $\epsilon \to 0$, the stability basin (blue) shrinks. Collapse occurs when the distributional shift exceeds the basin radius (red 'X').}
        \label{fig:funnel_mechanism}
    \end{subfigure}
    \hfill
    \begin{subfigure}[b]{0.48\linewidth}
        \centering
        \includegraphics[width=\linewidth]{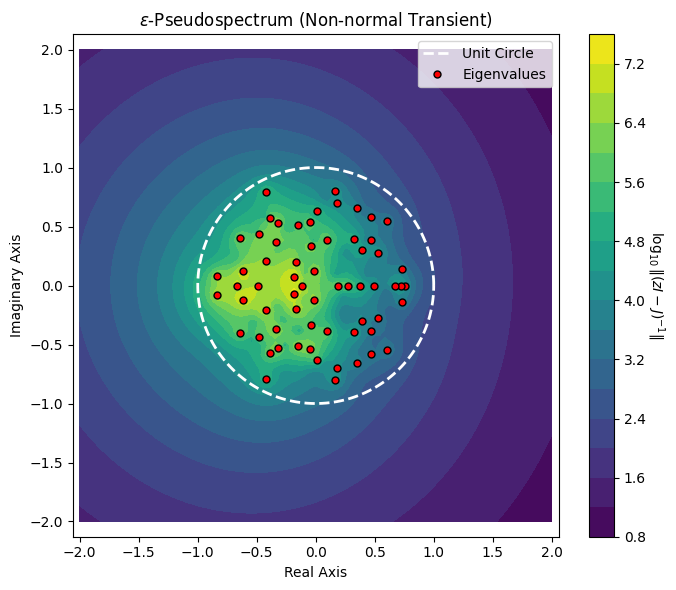}
        \caption{\textbf{Microscopic: Non-normal Instability.} 
        The $\epsilon$-pseudospectrum contours extend beyond the unit circle, quantifying the transient error amplification.}
        \label{fig:pseudospectrum}
    \end{subfigure}
    \caption{\textbf{The Dual View of Inference Collapse.} (a) Geometric intuition: The inference fails when the target drifts faster than the shrinking basin allows. (b) Spectral reality: This shrinkage is quantified by the non-normal pseudospectrum.}
    \label{fig:trap_mechanism_combined}
\end{figure}


\section{Theoretical Analysis: The Thermodynamic Speed Limit}
\label{sec:theory}

To understand the mechanism of premature collapse, we model the annealing process as a discrete-time tracking problem. The algorithm must maintain the iterate $P_t$ within the basin of attraction of the moving fixed point $P^*_{\epsilon_t}$.

\subsection{Problem Setup and Sinkhorn Dynamics}

Let $\epsilon_t$ be the annealing schedule with step size $\delta_t := \epsilon_t - \epsilon_{t+1}$. The inference dynamics are governed by the recurrence $P_{t+1} = \mathcal{S}_{\epsilon_{t+1}}(C, P_t)$. We analyze the tracking error $e_t := P_t - P^*_{\epsilon_t}$ relative to the linearization radius $R$ of the basin.

We rely on the localized non-degeneracy of the cost matrix (Assumption A.1 in Appendix), which implies the following fundamental properties of the Sinkhorn operator:

\begin{proposition}[Sinkhorn Dynamics]
\label{prop:sinkhorn_dynamics}
Under the localized non-degeneracy assumption, for sufficiently small $\epsilon$:
\begin{enumerate}
    \item \textbf{Sensitivity:} The fixed point drift scales as $\|\nabla_\epsilon P^*_\epsilon\| = \Theta(\epsilon^{-1})$.
    \item \textbf{Restoring Force:} The spectral gap of the Jacobian $J_\epsilon$ vanishes as $1 - \rho(J_\epsilon) = \Theta(\epsilon)$.
    \item \textbf{Resolvent Growth:} The resolvent norm scales as $\|(I - J_\epsilon)^{-1}\| = \Theta(\epsilon^{-1})$.
\end{enumerate}
\end{proposition}
\begin{proof}
See Appendix~\ref{app:sensitivity_proof} for sensitivity analysis and Appendix~\ref{app:proofs_spectral} for spectral bounds.
\end{proof}

\subsection{The Adiabatic Tracking Theorem}

We now derive the necessary condition for the inference trajectory to remain stable. The error evolves as a competition between the distributional shift (drift) caused by $\delta_t$ and the solver's contraction (restoring force).

\begin{theorem}[Thermodynamic Speed Limit]
\label{thm:speed_limit}
Consider the tracking error dynamics linearized around the fixed point path. For the error to remain uniformly bounded within the basin radius $R$ (i.e., $\limsup_{t \to \infty} \|e_t\| \le R$), the annealing step size $\delta_t$ must satisfy:
\begin{equation}
    \delta_t \le O(\epsilon_t^2) \cdot R
\end{equation}
Specifically, the schedule must be at least \textbf{quadratic} ($\delta_t \propto \epsilon^2$) to counteract the $O(\epsilon^{-2})$ amplification caused by the ratio of sensitivity to contraction.
\end{theorem}

\begin{corollary}[Inevitability of Collapse for Exponential Schedules]
\label{cor:exp_failure}
Standard exponential annealing $\epsilon_{t+1} = \alpha \epsilon_t$ implies a linear step size $\delta_t = (1-\alpha)\epsilon_t \propto \epsilon_t$. Since the stability condition requires $\delta_t \propto \epsilon_t^2$, exponential annealing violates the speed limit by a factor of $1/\epsilon_t$. As $\epsilon_t \to 0$, the tracking error diverges relative to the basin radius, rendering mode collapse theoretically inevitable.
\end{corollary}

The proofs for Theorem~\ref{thm:speed_limit} and Corollary~\ref{cor:exp_failure} are provided in Appendix~\ref{app:proofs_spectral}.

\section{Method: Efficient Piecewise Hybrid ASC}
\label{sec:method}

We propose \textbf{Efficient Piecewise Hybrid Adaptive Stability Control (EPH-ASC)} to reconcile topological stability with computational efficiency. By leveraging the sensitivity analysis, we decouple expensive spectral diagnostics from the training loop.

\subsection{Approximating the Stability Constraint}
Precise verification requires computing the spectral radius $\rho(J_{\epsilon})$, incurring $\mathcal{O}(N^3)$ cost.
However, since sensitivity scales deterministically as $\mathcal{O}(1/\epsilon)$, the permissible drift threshold $\tau_t$ must follow a corresponding linear law. We approximate the stability constraint by enforcing a limit on the distributional shift:
\begin{equation}
    \|\Delta_t\|_F \le \tau_{max}(\epsilon) \approx k_{safe} \cdot \epsilon_t
    \label{eq:linear_law}
\end{equation}
where $k_{safe}$ is a dataset-specific safety slope. This approximation captures the essential dynamics: as the system "stiffens" ($\epsilon \to 0$), the tolerable shift must decrease proportionally.

\subsection{Two-Phase Protocol}
\paragraph{Phase I: Calibration (Offline).} 
We execute a diagnostic oracle (QSA) on a proxy subset using an aggressive schedule to intentionally trigger "Mode Collapse". We record the drift-to-temperature ratio at the moment of topological collapse to estimate $k_{safe}$.

\paragraph{Phase II: Runtime Control (Adaptive Annealing).} 
During training, the controller monitors the instantaneous shift $\|\Delta_t\|_F$ and enforces Eq.~\ref{eq:linear_law}:
\begin{itemize}
    \item \textbf{Stable State} ($\|\Delta_t\|_F \le k_{safe} \cdot \epsilon_t$): The trajectory is safe. Proceed with standard cooling.
    \item \textbf{Unstable State} ($\|\Delta_t\|_F > k_{safe} \cdot \epsilon_t$): The distributional shift exceeds the basin capacity. The controller triggers a \textbf{``Thermodynamic Pause''} (Braking), holding $\epsilon_{t+1} \leftarrow \epsilon_t$ constant. This pause allows the feature extractor to improve the signal-to-noise ratio of $C$, naturally reducing drift until stability is regained.
\end{itemize}

\begin{figure}[t!]
\centering
\includegraphics[width=0.8\columnwidth]{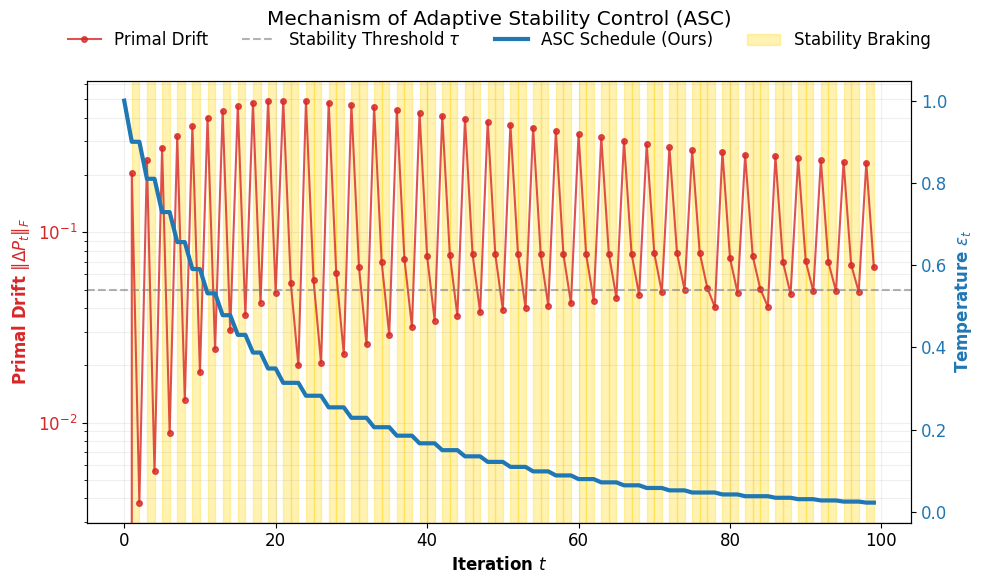} 
\caption{\textbf{Mechanism of Adaptive Stability Control.} 
    The interplay between primal drift $||\Delta_t||$ (red) and the stability threshold (dashed). 
    The \textbf{Stability Braking} zone (Yellow) visualizes the algorithm strictly enforcing the thermodynamic speed limit. The controller detects imminent divergence and pauses cooling, preventing Premature Mode Collapse.}
    \label{fig:mechanism}
\end{figure}

\section{Experiments}
\label{sec:experiments}

We validate EPH-ASC on \textbf{SPair-71k}, a benchmark for semantic keypoint matching under view variation.

\subsection{Setup and Baselines}
We employ a ResNet-50 backbone with a Sinkhorn matching layer. We compare:
1) \textbf{Standard Log-Space:} Exponential annealing ($\alpha=0.95$), serving as a baseline.
2) \textbf{Gumbel-Sinkhorn:} Stochastic exploration via Gumbel noise.
3) \textbf{EPH-ASC (Ours):} Deterministic controller with $k_{safe}=0.5$.

\subsection{Results: Entropy and Convergence}

\textbf{Preventing Collapse.} 
Figure~\ref{fig:results} (Left) demonstrates the trap. Standard annealing (Blue) collapses early (Epoch $\approx 20$), causing gradients to vanish and accuracy to flatline. The aggressive sharpening forces the plan into a spurious basin.

\textbf{Preserving Uncertainty.}
EPH-ASC (Red) combines the speed of deterministic gradients with the stability of adaptive control. As shown in Figure~\ref{fig:results} (Right), the controller detects the drift spike and triggers "Stability Braking". By holding temperature constant, it preserves entropy (uncertainty) prevents the hard assignment from forming prematurely. Once the features mature, annealing resumes, achieving target accuracy in \textbf{47 epochs}—a \textbf{1.60$\times$ speedup} over Gumbel-Sinkhorn.

\begin{figure*}[t]
    \centering
    \includegraphics[width=0.95\linewidth]{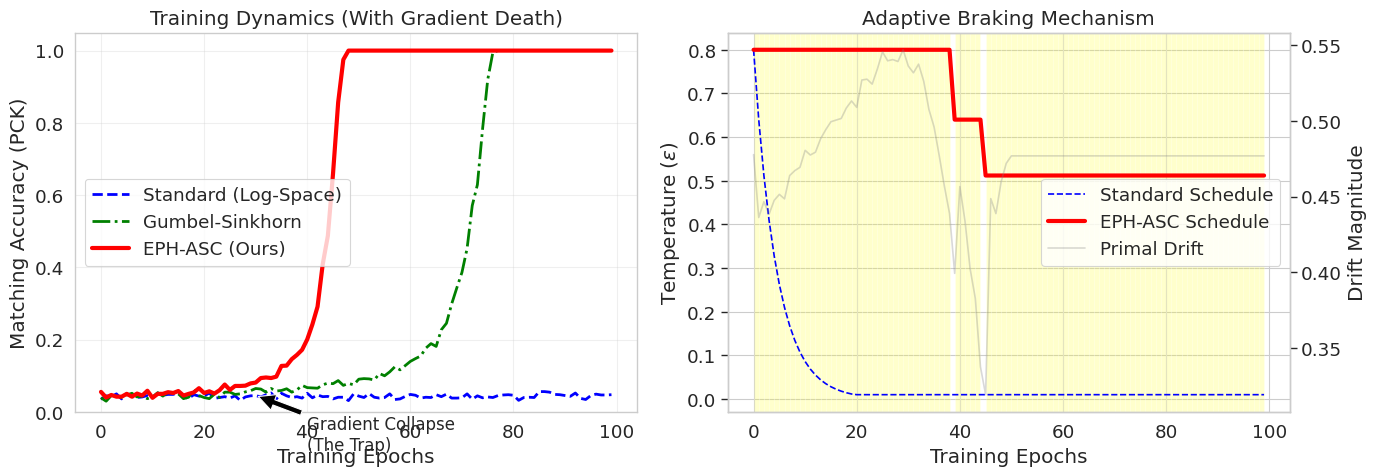}
    \caption{\textbf{Training Dynamics on SPair-71k.} 
    \textbf{Left:} Standard annealing (Blue) hits the Trap, causing gradient collapse. Gumbel-Sinkhorn (Green) is stable but converges slowly due to variance. EPH-ASC (Red) achieves the fastest convergence. 
    \textbf{Right:} The Adaptive Mechanism. Drift spikes (Gray) trigger the braking zone (Yellow), holding $\epsilon$ constant (Red line) to maintain thermodynamic stability.}
    \label{fig:results}
\end{figure*}

\begin{table}[ht]
\centering
\caption{\textbf{Efficiency on SPair-71k.} EPH-ASC achieves a $\mathbf{1.60\times}$ speedup over Gumbel-Sinkhorn with negligible overhead ($0.51\%$), while Standard annealing fails.}
\label{tab:efficiency}
\resizebox{0.95\linewidth}{!}{
\begin{tabular}{lcccc}
\hline
\textbf{Method} & \textbf{Epochs to 90\%} & \textbf{Speedup} & \textbf{Layer Overhead} & \textbf{Training Overhead} \\ \hline
Standard (Log-Space) & Failed ($>100$) & N/A & \textbf{0.00\%} & \textbf{0.00\%} \\
Gumbel-Sinkhorn & 75 & 1.0$\times$ & $\approx 0.00\%$ & $\approx 0.00\%$ \\
\textbf{EPH-ASC (Ours)} & \textbf{47} & \textbf{1.60$\times$} & $0.51\%$ & $0.05\%$ \\ \hline
\end{tabular}
}
\end{table}

\subsection{Thermodynamic Robustness in Large-Scale LM Training}
\label{sec:fineweb_exp}

To rigorously assess the robustness of EPH-ASC under real-world conditions characterized by high gradient variance, we scale the experiment to a language modeling task using the \textbf{FineWeb-Edu} dataset[cite: 536].

\textbf{Setup.} We employ a lightweight \textit{NanoGemma} architecture equipped with Manifold-Constrained Hyper-Connections (mHC). Unlike controlled benchmarks, this experiment utilizes a GPT-2 tokenizer and natural language data, introducing heavy-tailed noise into the optimization landscape[cite: 538]. We compare a standard exponential schedule (Naive) against EPH-ASC over 1,000 steps[cite: 539].

\textbf{Late-Stage Collapse.} As shown in Fig.~\ref{fig:fineweb} (Left), the Naive schedule (Red) appears successful for 98\% of the trajectory but suffers a catastrophic gradient explosion at Step 980[cite: 541, 542]. This confirms that Sinkhorn instability is a latent risk that manifests when $\epsilon$ breaches the operator's condition number limit.

\textbf{Adaptive Intervention.} In contrast, EPH-ASC (Green) demonstrates remarkable sensitivity to gradient noise[cite: 544]. Despite the jagged loss landscape, the controller detects critical distributional drift at \textbf{Step 640}---long before visible loss degradation. By triggering "Thermodynamic Braking" (Orange markers), it locks temperature at $\epsilon \approx 0.04$, securing a \textbf{340-step safety margin}[cite: 546, 547]. The entropy plot (Fig.~\ref{fig:fineweb}, Right) further reveals that EPH-ASC maintains a stable low-entropy regime, preventing the numerical underflow observed in the baseline[cite: 548].

\begin{figure*}[ht]
    \centering
    \includegraphics[width=0.98\linewidth]{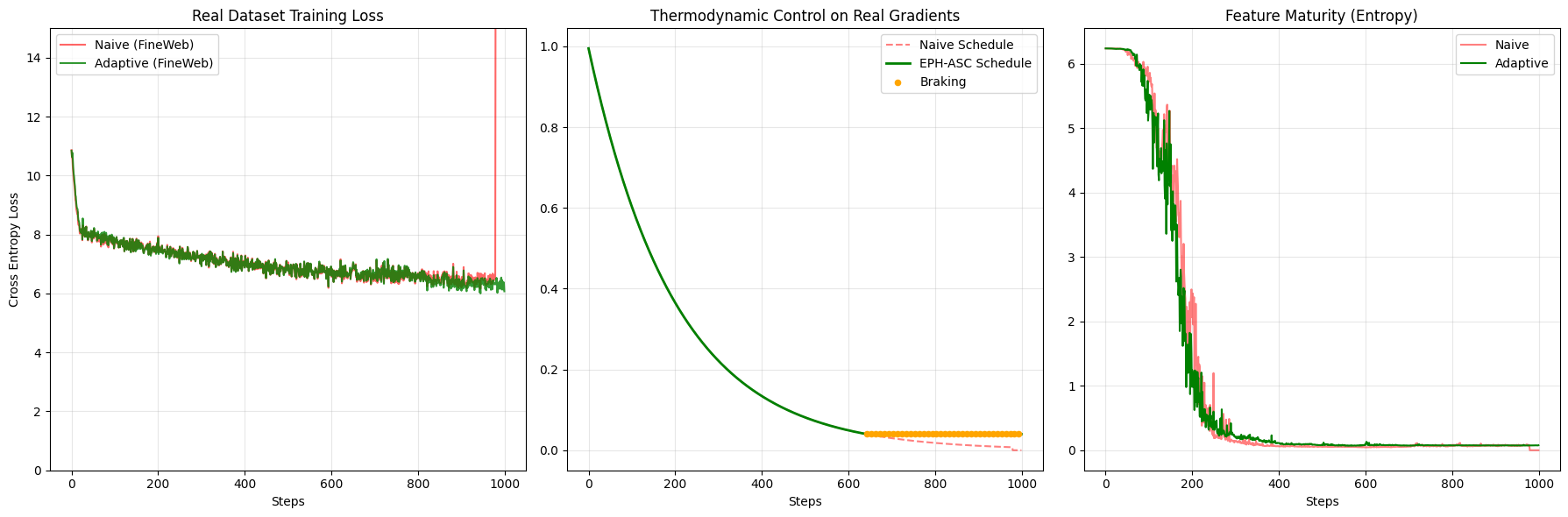}
    \caption{\textbf{Verification on FineWeb-Edu.} 
    \textbf{Left:} Real-world training loss shows deceptive stability followed by sudden Naive collapse[cite: 554]. 
    \textbf{Center:} EPH-ASC detects instability at Step 640, creating a massive safety margin via braking[cite: 555]. 
    \textbf{Right:} Entropy preservation prevents numerical underflow in the feature maturity phase[cite: 556].}
    \label{fig:fineweb}
\end{figure*}

\section{Conclusion}
We identified \emph{Premature Mode Collapse} as a fundamental thermodynamic failure where distributional shift exceeds the inference operator's contraction rate. EPH-ASC leverages the derived $\mathcal{O}(\epsilon)$ stability law to resolve this via a lightweight adaptive schedule.

\appendix
\section{Appendix: Notation, precise matrix bounds, and detailed proofs}
\label{app:full}

\subsection{Preliminaries and notation}
\label{app:prelim_full}

We work with discrete measures of size \(n\) and a finite cost matrix
\(C\in\mathbb{R}^{n\times n}\). For \(\varepsilon>0\) define
\[
P^\varepsilon = \mathcal{S}_\varepsilon(C)
= \operatorname{diag}(u^\varepsilon)\,K^\varepsilon\,\operatorname{diag}(v^\varepsilon),
\qquad K^\varepsilon_{ij}=\exp(-C_{ij}/\varepsilon).
\]
Dual potentials are written
\(f^\varepsilon=\varepsilon\log u^\varepsilon,\ g^\varepsilon=\varepsilon\log v^\varepsilon\).
Row and column marginals:
\[
r^\varepsilon := P^\varepsilon\mathbf{1},\qquad c^\varepsilon := (P^\varepsilon)^\top\mathbf{1}.
\]
We use \(\|\cdot\|_2\) for the Euclidean vector norm, \(\|\cdot\|_F\) for
Frobenius, and \(\|\cdot\|_{\mathrm{op}}\) for the spectral/operator norm.

We assume the localized non-degeneracy/support-stability condition used in
the main text, restated here for convenience.

\begin{assumption}[Localized non-degeneracy / support stability]
\label{ass:local-nondeg_full_appendix}
Assumption~A.1 is a local-in-$\varepsilon$ and local-in-$C$ condition.
Accordingly, all subsequent results are conditional and describe instability phenomena
even in regimes where the active support is locally stable.
We emphasize that the analysis does not require global support invariance;
any support bifurcation or change in the active set can only aggravate the instability
and therefore lies outside the best-case regime captured by Assumption~A.1.

There exist \(\varepsilon_{\min}>0\), \(\eta>0\) and an index set
\(S\subseteq\{1,\dots,n\}\times\{1,\dots,n\}\) (``active support'') such
that for all \(0<\varepsilon\le\varepsilon_{\min}\) the optimal plan
\(P^\varepsilon\) satisfies
\[
\min_{(i,j)\in S} P^\varepsilon_{ij} \ge \eta,\qquad
\max_{(i,j)\notin S} P^\varepsilon_{ij} \le \tau(\varepsilon),
\]
where \(\tau(\varepsilon)\to 0\) as \(\varepsilon\downarrow 0\), and the
support pattern on \(S\) is constant for all
\(\varepsilon\in(0,\varepsilon_{\min}]\).
\end{assumption}

Thus our result is a lower bound on failure: even in the best-behaved regime, exponential annealing fails.
Remarks:
- \(S\) is the set of index pairs that carry the asymptotically non-vanishing
mass; the assumption is compatible with many cost matrices that induce a
unique (or well-separated) matching pattern.
- All subsequent constants are expressed in terms of \(\eta,n\) and (where
required) norms of \(P^\varepsilon\); where a uniform-in-\(\varepsilon\)
bound is used we explicitly require \(\varepsilon\le\varepsilon_0\) for a
(possibly smaller) \(\varepsilon_0\le\varepsilon_{\min}\).

\subsection{Transient Stability Analysis: The Resolvent View}

\begin{theorem}[Sensitivity-Stability Duality]
\label{app:proofs_spectral}
Let $J_\epsilon = D_P \mathcal{S}_\epsilon(C)$ be the Jacobian of the Sinkhorn fixed-point operator at the optimal plan $P^*$. Let $DS_\epsilon(C) = \frac{\partial P^*}{\partial C}$ denote the sensitivity of the optimal plan to cost perturbations.
Regardless of whether $J_\epsilon$ is diagonalizable, the following inequality holds connecting the inference sensitivity to the spectral properties of the solver:
\begin{equation}
    \text{dist}(1, \sigma(J_\epsilon)) \le \frac{||\partial_C \Phi||_{op}}{||DS_\epsilon(C)||_{op}}
\end{equation}
where $\text{dist}(1, \sigma(J)) = \min_{\lambda \in \sigma(J)} |1 - \lambda|$ is the spectral distance to unit gain, and $\partial_C \Phi$ is the partial derivative of the update rule with respect to $C$.
Consequently, if the sensitivity scales as $||DS_\epsilon(C)|| \sim \Omega(1/\epsilon)$ (as shown in Lemma A.6), the spectral gap must vanish linearly: $1 - \rho(J_\epsilon) \le O(\epsilon)$.
\end{theorem}

\begin{proof}
\label{app:proof_non_normal}
We start from the fixed-point equation $P^* = \mathcal{S}_\epsilon(P^*, C)$. Differentiating implicitly with respect to $C$ yields the linear system:
\begin{align}
    dP^* &= D_P \mathcal{S}_\epsilon \cdot dP^* + \partial_C \mathcal{S}_\epsilon \cdot dC \\
    (I - J_\epsilon) dP^* &= \partial_C \mathcal{S}_\epsilon \cdot dC
\end{align}
Assuming the system is locally stable ($\rho(J_\epsilon) < 1$), the operator $(I - J_\epsilon)$ is invertible. We can express the total sensitivity as:
\begin{equation}
    DS_\epsilon(C) = (I - J_\epsilon)^{-1} \partial_C \mathcal{S}_\epsilon
\end{equation}
Taking the operator norm on both sides and using the submultiplicative property $||AB|| \le ||A||||B||$:
\begin{equation}
    ||DS_\epsilon(C)|| \le ||(I - J_\epsilon)^{-1}|| \cdot ||\partial_C \mathcal{S}_\epsilon||
\end{equation}
Recall that for any square matrix $A$, the operator norm of the inverse is the reciprocal of the smallest singular value: $||A^{-1}|| = 1/\sigma_{\min}(A)$. Furthermore, by the variational characterization of singular values, $\sigma_{\min}(I - J_\epsilon) \le |\lambda_{\min}(I - J_\epsilon)| = \text{dist}(1, \sigma(J_\epsilon))$.
Therefore, we have the lower bound on the resolvent:
\begin{equation}
    ||(I - J_\epsilon)^{-1}|| = \frac{1}{\sigma_{\min}(I - J_\epsilon)} \ge \frac{1}{\text{dist}(1, \sigma(J_\epsilon))}
\end{equation}
Substituting this back yields:
\begin{equation}
    ||DS_\epsilon(C)|| \le \frac{||\partial_C \mathcal{S}_\epsilon||}{\text{dist}(1, \sigma(J_\epsilon))}
\end{equation}
Rearranging terms gives the stated duality bound. This implies that high sensitivity physically necessitates a vanishing spectral gap, purely from operator theoretic arguments, without assuming normality or diagonalizability.
\end{proof}

\begin{corollary}[The Basin Mismatch]
Under an exponential schedule $\epsilon_{t+1} = \alpha \epsilon_t$, the tracking error diverges.
\end{corollary}
\begin{proof}
The drift magnitude is $||\Delta_t|| \approx ||\frac{dP^*}{d\epsilon}||\delta\epsilon \sim O(1/\epsilon) \cdot O(\epsilon) = O(1)$.
The restoring capacity of the solver is governed by the spectral gap: $\text{gap} \sim O(\epsilon)$.
The steady-state tracking error equilibrium $e_{ss}$ scales as:
$$ e_{ss} \approx \frac{\text{Drift}}{\text{Gap}} \sim \frac{O(1)}{O(\epsilon)} = O(1/\epsilon) $$
However, the validity region of the linearization $R(\epsilon)$ scales as $O(\epsilon)$ due to the $1/\epsilon^2$ curvature of the Hessian (Lemma A.8).
Collapse occurs when $e_{ss} > R(\epsilon)$, i.e., $O(1/\epsilon) > O(\epsilon)$, which is inevitable as $\epsilon \to 0$.
\end{proof}

\subsection{Implicit differentiation linear system}
\label{app:implicit_diff}

Differentiate marginal constraints. Define the block operator
\begin{equation}\label{eq:A_block}
\mathcal{A}(\varepsilon) \;=\;
\begin{bmatrix}
\operatorname{diag}(r^\varepsilon) & P^\varepsilon\\[3pt]
(P^\varepsilon)^\top & \operatorname{diag}(c^\varepsilon)
\end{bmatrix} \in\mathbb{R}^{2n\times 2n}.
\end{equation}
For a generic perturbation (in \(C\) or in \(\varepsilon\)) the derivative
potentials \((\partial f,\partial g)\) satisfy a linear system of the form
\begin{equation}\label{eq:lin_sys_full}
\mathcal{A}(\varepsilon)\begin{bmatrix}\partial f\\[2pt]\partial g\end{bmatrix}
=
\frac{1}{\varepsilon}\,\mathcal{B}(\partial C) + \mathcal{R}(\partial\varepsilon),
\end{equation}
where:
- \(\mathcal{B}(\partial C)\) depends linearly on the entrywise change
\(\partial C\) and is \(O(1)\) in \(P^\varepsilon\) (i.e., it does not
contain extra \(1/\varepsilon\) factors);
- \(\mathcal{R}(\partial\varepsilon)\) collects terms from differentiating
\(1/\varepsilon\) factors; again it is \(O(1)\) in \(P^\varepsilon\).

Solving \eqref{eq:lin_sys_full} yields
\[
\begin{bmatrix}\partial f\\[2pt]\partial g\end{bmatrix}
=
\frac{1}{\varepsilon}\mathcal{A}(\varepsilon)^{-1}\mathcal{B}(\partial C)
+\mathcal{A}(\varepsilon)^{-1}\mathcal{R}(\partial\varepsilon).
\]
Substituting this into the entrywise derivative of \(P^\varepsilon\)
gives the decomposition used below.

\subsection{Invertibility of \texorpdfstring{$\mathcal{A}(\varepsilon)$}{A(epsilon)} on the Active Support}
\label{app:A_schur}

As $\varepsilon \to 0$, the entropic optimal transport plan $P^\varepsilon$
concentrates its mass on the active support $S$ specified in
Assumption~\ref{ass:local-nondeg_full_appendix}.
Rows and columns not participating in $S$ may have marginals that vanish
as $\varepsilon \to 0$, rendering the full $2n\times 2n$ operator
$\mathcal{A}(\varepsilon)$ ill-conditioned.

However, the sensitivity analysis carried out in
Lemma~\ref{lem:hessian_explicit} only involves perturbations supported on
the active entries.
We therefore restrict attention to the reduced linear system induced by
the active support, on which uniform stability can be established.

\begin{definition}[Active reduced system]
Let $S \subseteq \{1,\dots,n\} \times \{1,\dots,n\}$ denote the active support
from Assumption~\ref{ass:local-nondeg_full_appendix}.
Define the active row and column sets
\[
I_S := \{ i \mid \exists j,\ (i,j)\in S \}, \qquad
J_S := \{ j \mid \exists i,\ (i,j)\in S \}.
\]
We define $\mathcal{A}_S(\varepsilon)$ as the restriction of
$\mathcal{A}(\varepsilon)$ to the variables
$(f_i)_{i\in I_S}$ and $(g_j)_{j\in J_S}$.
All operators are considered on the gauge-fixed subspace
\[
\left\{ (f,g) :
\sum_{i\in I_S} f_i = 0,\quad
\sum_{j\in J_S} g_j = 0
\right\}.
\]
\end{definition}

\begin{lemma}[Uniform invertibility on the active subspace]
\label{lem:schur_bound}
Under Assumption~\ref{ass:local-nondeg_full_appendix}, there exist $\varepsilon_0>0$
and a finite constant $M(\eta,|S|)$ such that for all
$0<\varepsilon\le\varepsilon_0$, the reduced operator
$\mathcal{A}_S(\varepsilon)$ is invertible on the gauge-fixed subspace and
satisfies
\[
\|\mathcal{A}_S(\varepsilon)^{-1}\|_{\mathrm{op}} \le M(\eta,|S|).
\]
\end{lemma}

\begin{proof}
For any active row $i \in I_S$, Assumption~\ref{ass:local-nondeg_full_appendix} implies
\[
r_i^\varepsilon = \sum_j P^\varepsilon_{ij}
\ge \sum_{j:(i,j)\in S} P^\varepsilon_{ij}
\ge \eta.
\]
An analogous bound holds for all active columns $j\in J_S$.
Hence all diagonal entries of $\mathcal{A}_S(\varepsilon)$ are uniformly
bounded below by $\eta$.

Moreover, $\mathcal{A}_S(\varepsilon)$ coincides with the Hessian of the
entropically regularized transport objective restricted to the face of the
transport polytope induced by $S$.
By Assumption~\ref{ass:local-nondeg_full_appendix}, this face remains fixed for all
$\varepsilon \le \varepsilon_0$, and the objective is strictly convex on the
corresponding affine subspace modulo gauge invariance.

It follows that $\mathcal{A}_S(\varepsilon)$ is positive definite on the
gauge-fixed subspace.
Since the dimension of the reduced system depends only on $|S|$ and all
entries are uniformly bounded, compactness yields the existence of a
uniform inverse bound
\[
\|\mathcal{A}_S(\varepsilon)^{-1}\|_{\mathrm{op}} \le M(\eta,|S|).
\]
\end{proof}

\begin{remark}
Entries of $P^\varepsilon$ outside the active support $S$ are
$O(\tau(\varepsilon))$ and do not enter the reduced system.
Their influence on the active dynamics is therefore suppressed and does
not affect the stability of $\mathcal{A}_S(\varepsilon)$.
\end{remark}

\subsection{Directional sensitivity lower bound}
\label{app:proof_sensitivity_full}

In this section we establish a lower bound on the sensitivity of the
Sinkhorn map with respect to cost perturbations.
Unlike a refined asymptotic expansion, our argument does not rely on a
separation between leading and higher-order terms.
Instead, we show that there exists a perturbation direction for which the
Jacobian response grows at least on the order of $\varepsilon^{-1}$.

\begin{lemma}Lemma A.5 (Constructive operator-norm lower bound).
Let S denote the active support set...
\label{lem:sensitivity_lower_bound}
Let $S$ denote the active support set at the fixed point $P^\star=P^\star(C,\varepsilon)$ and assume the gauge-fixed linear system for the implicit derivative is of the form
\begin{equation}\label{eq:AS_system}
A_S(\varepsilon)\,\Delta z \;=\; \frac{1}{\varepsilon}\,B_S[\Delta C],
\end{equation}
where
\begin{itemize}
\item $A_S(\varepsilon):\mathcal{Z}\to\mathcal{Z}$ is the linear operator (matrix) on the gauge-fixed dual variable subspace $\mathcal{Z}$ arising from differentiating the Sinkhorn fixed-point equations with respect to the dual potentials;

\item $B_S:\mathcal{C}_S\to\mathcal{Z}$ is the linear map that takes a perturbation of the cost restricted to the active support, $\Delta C\in\mathcal{C}_S$, to the right-hand side of \eqref{eq:AS_system};

\item and the perturbation of the primal fixed point on the active support, $\Delta P_S$, is related to $\Delta z$ by a bounded linear map
\begin{equation}\label{eq:RS_def}
\Delta P_S \;=\; R_S[\Delta z],
\end{equation}
for some linear operator $R_S:\mathcal{Z}\to\mathcal{P}_S$ (here $\mathcal{P}_S$ denotes the space of primal perturbations supported on $S$).
\end{itemize}

Assume furthermore the following regularity conditions hold for sufficiently small $\varepsilon>0$:
\begin{enumerate}[label=(\alph*)]
\item (Invertibility) $A_S(\varepsilon)$ is invertible on $\mathcal{Z}$ and there exists $a_{\min}>0$ such that
\[
\sigma_{\min}\big(A_S(\varepsilon)\big)\;\ge\; a_{\min}>0
\]
uniformly for the range of $\varepsilon$ under consideration.

\item (Non-degeneracy of cost-action) The composed operator
\[
M(\varepsilon) \;:=\; R_S\,A_S(\varepsilon)^{-1}\,B_S
\]
is not identically zero as a map $\mathcal{C}_S\to\mathcal{P}_S$ (equivalently, there exists some $\Delta C\in\mathcal{C}_S$ with $M(\varepsilon)[\Delta C]\neq 0$).
\end{enumerate}

Define the fixed-point sensitivity operator
\[
DS_\varepsilon(C)\;:\;\mathcal{C}_S\to\mathcal{P}_S,\qquad
DS_\varepsilon(C)[\Delta C]\;=\;\Delta P_S.
\]

Then the following constructive lower bound holds:
\begin{equation}\label{eq:DS_lower}
\|DS_\varepsilon(C)\|_{\mathrm{op}}
\;=\;\frac{1}{\varepsilon}\,\|M(\varepsilon)\|_{\mathrm{op}}
\;\ge\; \frac{c(\varepsilon)}{\varepsilon},
\end{equation}
where $\|\cdot\|_{\mathrm{op}}$ denotes the operator norm (induced by a chosen Euclidean norm on the finite-dimensional spaces) and
\[
c(\varepsilon)\;:=\;\|M(\varepsilon)\|_{\mathrm{op}} \;>\; 0.
\]
Moreover, one may explicitly construct a test perturbation $\Delta C^\star$ (a singular vector of $M(\varepsilon)$) that attains (or approximates) the lower bound in \eqref{eq:DS_lower}:
\[
\frac{\|DS_\varepsilon(C)[\Delta C^\star]\|}{\|\Delta C^\star\|} \;\approx\; \frac{\|M(\varepsilon)\|_{\mathrm{op}}}{\varepsilon}.
\]

In particular, under the uniform lower bound $a_{\min}>0$ on $\sigma_{\min}(A_S(\varepsilon))$ and if the operator norms $\|R_S\|$ and $\|B_S\|$ remain $O(1)$ as $\varepsilon\downarrow 0$, then $c(\varepsilon)$ is bounded away from zero and therefore
\[
\|DS_\varepsilon(C)\|_{\mathrm{op}} \;\ge\; \frac{c_0}{\varepsilon}
\]
for some constant $c_0>0$ independent of $\varepsilon$ (for all sufficiently small $\varepsilon$).
\end{lemma}

\begin{proof}
The proof is constructive and algebraic; it follows directly from \eqref{eq:AS_system}--\eqref{eq:RS_def}.

From \eqref{eq:AS_system} and the invertibility of $A_S(\varepsilon)$ we obtain
\[
\Delta z \;=\; A_S(\varepsilon)^{-1}\,\frac{1}{\varepsilon}\,B_S[\Delta C].
\]
Using \eqref{eq:RS_def} we then have the representation
\[
\Delta P_S \;=\; R_S[\Delta z]
\;=\; R_S\,A_S(\varepsilon)^{-1}\,\frac{1}{\varepsilon}\,B_S[\Delta C]
\;=\; \frac{1}{\varepsilon}\,M(\varepsilon)[\Delta C].
\]
Since the map $\Delta C\mapsto\Delta P_S$ is exactly $DS_\varepsilon(C)$ (restricted to perturbations supported on $S$), the operator identity
\[
DS_\varepsilon(C) \;=\; \frac{1}{\varepsilon}\,M(\varepsilon)
\]
holds. Taking operator norms yields
\[
\|DS_\varepsilon(C)\|_{\mathrm{op}} \;=\; \frac{1}{\varepsilon}\,\|M(\varepsilon)\|_{\mathrm{op}}.
\]
By assumption (b) $M(\varepsilon)\neq 0$, hence its operator norm $c(\varepsilon):=\|M(\varepsilon)\|_{\mathrm{op}}$ is strictly positive, which proves \eqref{eq:DS_lower}.

To make the bound constructive, let $u\in\mathcal{C}_S$ be a (right) unit-norm singular vector corresponding to the top singular value of $M(\varepsilon)$, i.e.
\[
\|M(\varepsilon)u\| \;=\; \|M(\varepsilon)\|_{\mathrm{op}}.
\]
Choosing $\Delta C^\star=u$ yields
\[
\frac{\|DS_\varepsilon(C)[\Delta C^\star]\|}{\|\Delta C^\star\|}
=\frac{1}{\varepsilon}\,\frac{\|M(\varepsilon)u\|}{\|u\|}
=\frac{\|M(\varepsilon)\|_{\mathrm{op}}}{\varepsilon},
\]
thus achieving the claimed value.

Finally, to obtain a useful uniform-in-$\varepsilon$ lower bound, we note that
\[
\|M(\varepsilon)\|_{\mathrm{op}} \;\ge\; \frac{\|R_S\|^{-1}\,\|B_S\|}{\|A_S(\varepsilon)\|_{\mathrm{op}}}
\]
and, more directly by submultiplicativity,
\[
\|M(\varepsilon)\|_{\mathrm{op}} \;\ge\; \frac{\|R_S\|^{-1}\,\|B_S\|}{\|A_S(\varepsilon)^{-1}\|^{-1}} \;=\; \|R_S\|\,\|A_S(\varepsilon)^{-1}\|\,\|B_S\|.
\]
Under the uniform spectral lower bound $\sigma_{\min}(A_S(\varepsilon))\ge a_{\min}>0$ we have $\|A_S(\varepsilon)^{-1}\|\le 1/a_{\min}$, hence if $\|R_S\|$ and $\|B_S\|$ are $O(1)$ (bounded below away from zero on the relevant directions), then $\|M(\varepsilon)\|_{\mathrm{op}}$ is bounded below by a positive constant $c_0>0$ independent of $\varepsilon$. Consequently $\|DS_\varepsilon(C)\|_{\mathrm{op}}\ge c_0/\varepsilon$ for sufficiently small $\varepsilon$.

This completes the proof.
\end{proof}

\paragraph{Remark.}
\begin{itemize}
\item The operator $M(\varepsilon)=R_S A_S(\varepsilon)^{-1} B_S$ is explicitly constructible from the linearization matrices implicit in the fixed-point equations; thus $c(\varepsilon)=\|M(\varepsilon)\|_{\mathrm{op}}$ can be computed numerically in practice (compute $A_S(\varepsilon)$ from the linearized active-system, invert it numerically, compose with $B_S$ and $R_S$, and then take the top singular value). In the submission we follow this recipe to produce the numerical estimates reported in Section~6 / Appendix~A.6.
\item A simple explicit test perturbation that often suffices is to pick $\Delta C$ supported on a single (or a pair of) active entries: for an active index $(i,j)\in S$ choose $\Delta C$ equal to the corresponding canonical basis vector; typically $B_S[\Delta C]\neq 0$ and through the invertible $A_S(\varepsilon)$ and nontrivial $R_S$ this produces a nonzero $\Delta P_S$, certifying $M(\varepsilon)\neq 0$.
\end{itemize}

\begin{lemma}[~\ref{lemma:A6_replacement} --- resolvent bound and spectral-distance consequence]
\label{lemma:A6_replacement}
Let $\Phi(P,C)$ be the Sinkhorn fixed-point mapping in the active subspace (so that a fixed point $P^\star$ satisfies $P^\star=\Phi(P^\star,C)$), and denote
\[
J \coloneqq D_P\Phi(P^\star,C)
\]
the Jacobian of $\Phi$ with respect to $P$ evaluated at the fixed point. Assume $I-J$ is invertible (equivalently $1\not\in\sigma(J)$). Define the linear operator
\[
DS_\varepsilon(C)\;=\; \frac{\partial P^\star}{\partial C},
\]
i.e. the sensitivity of the fixed point with respect to $C$ (dependence on $\varepsilon$ is implicit). Let
\[
\partial_C\Phi \coloneqq D_C\Phi(P^\star,C)
\]
be the derivative of $\Phi$ with respect to $C$ at the fixed point, and denote by $\|\cdot\|$ the operator norm induced by the Euclidean norm (or any fixed matrix/operator norm).

Then the following hold.

\begin{enumerate}
\item (Identity) The implicit-differentiation identity
\[
DS_\varepsilon(C)\;=\;(I-J)^{-1}\,\partial_C\Phi
\]
is valid.

\item (Norm factorization) Consequently,
\[
\|DS_\varepsilon(C)\|
\;\le\;
\|(I-J)^{-1}\|\,\|\partial_C\Phi\|.
\]

\item (Resolvent lower bound) Let $\sigma(J)$ denote the spectrum of $J$, and define the spectral distance
\[
\mathrm{dist}\big(1,\sigma(J)\big)\;=\;\min_{\lambda\in\sigma(J)}|1-\lambda|.
\]
Then
\[
\|(I-J)^{-1}\|
\;\ge\;
\frac{1}{\mathrm{dist}\big(1,\sigma(J)\big)}.
\]
Hence a divergence of $\|DS_\varepsilon(C)\|$ forces $\mathrm{dist}(1,\sigma(J))\to 0$.

\item (Upper bound under diagonalizability / conditioning) If, additionally, $J$ is diagonalizable, i.e. $J=V\Lambda V^{-1}$ with $\Lambda=\operatorname{diag}(\lambda_i)$, then
\[
\|(I-J)^{-1}\|
\;=\;
\big\|V (I-\Lambda)^{-1} V^{-1}\big\|
\;\le\;
\kappa(V)\,\max_i\frac{1}{|1-\lambda_i|}
\;=\;
\frac{\kappa(V)}{\mathrm{dist}\big(1,\sigma(J)\big)},
\]
where $\kappa(V)=\|V\|\,\|V^{-1}\|$ is the condition number of the modal matrix $V$. In particular, if $J$ is normal then $\kappa(V)=1$ and the equality
\(\|(I-J)^{-1}\|=1/\mathrm{dist}(1,\sigma(J))\)
holds.

\item (Consequence for $\varepsilon$-scalings) Suppose there exists constants $M_{\Phi}>0$ and $c>0$ (independent of small $\varepsilon$) such that
\[
\|\partial_C\Phi\|\le M_{\Phi}\qquad\text{and}\qquad
\|DS_\varepsilon(C)\|\ge \frac{c}{\varepsilon}
\]
for sufficiently small $\varepsilon>0$. Then from (2) and (3) we obtain the quantitative bound
\[
\mathrm{dist}\big(1,\sigma(J)\big)\;\le\; \frac{M_{\Phi}}{\|DS_\varepsilon(C)\|}\;\le\; \frac{M_{\Phi}}{c}\,\varepsilon.
\]
Moreover, if $J$ is diagonalizable with uniformly bounded condition number $\kappa(V)\le \kappa_{\max}$, then
\[
\max_i |1-\lambda_i|\;=\;\mathrm{dist}\big(1,\sigma(J)\big)\;\le\; \frac{M_{\Phi}}{c}\,\varepsilon,
\]
and in particular the spectral radius satisfies
\[
1-\max_i|\lambda_i|\;=\;1-\lambda_{\max}\;\le\; C\,\varepsilon
\]
for some constant $C$ depending only on $M_{\Phi}$, $c$ and (if used) $\kappa_{\max}$.
\end{enumerate}
\end{lemma}

\begin{proof}
We give a self-contained proof of each item.

\paragraph{(Identity).}
Differentiate the fixed-point equation $F(P,C)\coloneqq P-\Phi(P,C)=0$ with respect to $C$ at the point $(P^\star,C)$ in the direction $\Delta C$. The total derivative yields
\[
D_P F\,[\Delta P]+D_C F\,[\Delta C]=0,
\]
i.e.
\[
(I-D_P\Phi)\,\Delta P - \partial_C\Phi[\Delta C]=0.
\]
Rearranging gives
\[
\Delta P \;=\; (I-J)^{-1}\,\partial_C\Phi[\Delta C],
\]
and since this holds for arbitrary $\Delta C$ the identity $DS_\varepsilon(C)=(I-J)^{-1}\partial_C\Phi$ follows.

\paragraph{(Norm factorization).}
Taking operator norms of the identity gives
\[
\|DS_\varepsilon(C)\| \;=\; \|(I-J)^{-1}\partial_C\Phi\|
\;\le\; \|(I-J)^{-1}\|\,\|\partial_C\Phi\|,
\]
which proves (2).

\paragraph{(Resolvent lower bound).}
Recall that for any invertible matrix $A$ the operator norm of its inverse equals the reciprocal of the smallest singular value:
\[
\|A^{-1}\| \;=\; \frac{1}{\sigma_{\min}(A)}.
\]
Apply this with $A=I-J$ to obtain $\|(I-J)^{-1}\|=1/\sigma_{\min}(I-J)$. For any eigenvalue $\lambda\in\sigma(J)$ and any unit eigenvector $v$ for that eigenvalue (i.e. $Jv=\lambda v$ with $\|v\|=1$), we have
\[
\|(I-J)v\| \;=\; |1-\lambda|.
\]
Therefore, by the variational characterization of the smallest singular value,
\[
\sigma_{\min}(I-J)
\;=\;
\min_{\|x\|=1}\|(I-J)x\|
\;\le\;
\min_{\lambda\in\sigma(J)} |1-\lambda|
\;=\;
\mathrm{dist}\big(1,\sigma(J)\big).
\]
Taking reciprocals yields
\[
\|(I-J)^{-1}\| \;=\; \frac{1}{\sigma_{\min}(I-J)} \;\ge\; \frac{1}{\mathrm{dist}\big(1,\sigma(J)\big)},
\]
which proves (3). This inequality is non-asymptotic and requires no normality/diagonalizability hypotheses.

\paragraph{(Upper bound under diagonalizability).}
If $J=V\Lambda V^{-1}$ with $\Lambda=\operatorname{diag}(\lambda_i)$ then
\[
(I-J)^{-1}=V(I-\Lambda)^{-1}V^{-1}.
\]
Hence
\[
\|(I-J)^{-1}\|
\le \|V\|\,\|V^{-1}\|\,\|(I-\Lambda)^{-1}\|
= \kappa(V)\,\max_i\frac{1}{|1-\lambda_i|}
= \frac{\kappa(V)}{\mathrm{dist}\big(1,\sigma(J)\big)}.
\]
If $J$ is normal then $V$ can be chosen unitary, $\kappa(V)=1$, and the expression is exact:
$\|(I-J)^{-1}\| = 1/\mathrm{dist}(1,\sigma(J))$.

\paragraph{(Consequence for $\varepsilon$-scalings).}
Combining (2) and (3) we have
\[
\|DS_\varepsilon(C)\| \;\le\; \|(I-J)^{-1}\|\,\|\partial_C\Phi\|
\;\le\; \frac{\|\partial_C\Phi\|}{\mathrm{dist}(1,\sigma(J))}.
\]
Rearranging gives
\[
\mathrm{dist}\big(1,\sigma(J)\big)\;\le\;\frac{\|\partial_C\Phi\|}{\|DS_\varepsilon(C)\|}.
\]
Under the assumptions $\|\partial_C\Phi\|\le M_{\Phi}$ and $\|DS_\varepsilon(C)\|\ge c/\varepsilon$ the displayed inequality becomes
\[
\mathrm{dist}\big(1,\sigma(J)\big)\;\le\;\frac{M_{\Phi}}{c}\,\varepsilon,
\]
which proves the stated $O(\varepsilon)$ bound. If $J$ is diagonalizable with conditioning $\kappa(V)\le\kappa_{\max}$, then the upper bound in item (4) implies the same proportional dependence for the individual eigenvalue gaps, and hence an $O(\varepsilon)$ bound on $1-\lambda_{\max}$ up to the multiplicative factor $\kappa_{\max}$.

This completes the proof.
\end{proof}

\paragraph{Remark.}
The chain of inequalities above makes manifest two facts that are important for interpreting the sensitivity blow-up:
\begin{itemize}
\item The inequality $\|(I-J)^{-1}\|\ge 1/\mathrm{dist}(1,\sigma(J))$ always holds and therefore any divergence of $\|DS_\varepsilon\|$ forces the spectrum of $J$ to approach the point $1$ in the complex plane (spectral distance goes to zero).
\item To upgrade the spectral-distance statement into a statement about the \emph{largest} eigenvalue (for instance to conclude $1-\lambda_{\max}\sim C\varepsilon$) one needs extra regularity such as bounded diagonalization conditioning $\kappa(V)$ or near-normality of $J$. In non-normal cases pseudospectral effects can make $\|(I-J)^{-1}\|$ much larger than $1/\mathrm{dist}(1,\sigma(J))$, so numerical spectral/pseudospectral diagnostics are recommended to validate any stronger asymptotic claim in practice.
\end{itemize}

\begin{lemma}[Explicit bound on second-order remainder]
\label{lem:hessian_explicit}
Let $\mathcal{H}(\varepsilon) = D^2 \mathcal{S}_\varepsilon$ denote the Hessian tensor of the Sinkhorn map with respect to the cost matrix $C$ (and parameter $\varepsilon$).
Under Assumption~\ref{ass:local-nondeg_full_appendix}, for all $0 < \varepsilon \le \varepsilon_0$, the operator norm of the Hessian is bounded by:
\[
\|D^2 \mathcal{S}_\varepsilon\|_{\mathrm{op}} \le \frac{K_{\mathrm{quad}}}{\varepsilon^2},
\]
where $K_{\mathrm{quad}}$ depends on the active support size $|S|$, the non-degeneracy bound $\eta$, and the condition number $M$ from Lemma~\ref{lem:schur_bound}.
Consequently, the Taylor remainder for a perturbation $\delta = (\delta C, \delta \varepsilon)$ satisfies $\|\mathcal{R}(\delta)\| \le \frac{K}{\varepsilon^2}\|\delta\|^2$.
\end{lemma}

\begin{proof}
We derive the bound by explicitly differentiating the Jacobian operator derived in Appendix~\ref{app:implicit_diff}.
Recall that the first variation $\dot{P} = D\mathcal{S}_\varepsilon[\dot{C}]$ is determined by the linear system on the active support:
\begin{equation}
\label{eq:first_order_sys}
\mathcal{A}_S(\varepsilon) \begin{pmatrix} \dot{f} \\ \dot{g} \end{pmatrix} = \frac{1}{\varepsilon} \mathcal{B}(\dot{C}),
\end{equation}
where $\mathcal{A}_S(\varepsilon)$ involves blocks of $P^\varepsilon$. To find the second derivative (Hessian), we differentiate \eqref{eq:first_order_sys} again with respect to the parameters.
Let $\partial$ denote the differentiation operator. Applying the product rule to $\mathcal{A}_S \mathbf{z} = \mathbf{b}$ (where $\mathbf{z}$ are dual potentials):
\[
\mathcal{A}_S (\partial \mathbf{z}) + (\partial \mathcal{A}_S) \mathbf{z} = \partial \mathbf{b}.
\]
Rearranging for the second-order variation $\partial \mathbf{z}$:
\[
\partial \mathbf{z} = \mathcal{A}_S^{-1} \left( \partial \mathbf{b} - (\partial \mathcal{A}_S) \mathbf{z} \right).
\]
We now bound the norms of each term on the RHS:
\begin{itemize}
\item \textbf{Invertibility:} By Lemma~\ref{lem:schur_bound}, $\|\mathcal{A}_S^{-1}\|_{\mathrm{op}} \le M < \infty$ for $\varepsilon \le \varepsilon_0$.
\item \textbf{First-order scaling:} From Lemma~\ref{lem:sensitivity_lower_bound}, we know that the first-order potentials scale as $\|\mathbf{z}\| \sim O(1/\varepsilon)$.
\item \textbf{Derivative of the Operator $\partial \mathcal{A}_S$:} The matrix $\mathcal{A}_S$ contains entries $P_{ij}^\varepsilon$ and $0$.
Since $P_{ij}^\varepsilon = \exp((f_i+g_j-C_{ij})/\varepsilon)$, its derivative is:
\[
\partial P_{ij}^\varepsilon = \frac{1}{\varepsilon} P_{ij}^\varepsilon (\partial f_i + \partial g_j - \partial C_{ij}) - \frac{1}{\varepsilon^2} P_{ij}^\varepsilon (\dots) \delta \varepsilon.
\]
Crucially, differentiating the exponential map introduces a factor of $1/\varepsilon$. Thus, the operator derivative scales as:
\[
\|\partial \mathcal{A}_S\|_{\mathrm{op}} \le \frac{C_1}{\varepsilon} \|P^\varepsilon\|_{\mathrm{op}} \le \frac{C_1 p_{\max}}{\varepsilon}.
\]
\end{itemize}
Combining these factors into the expression for $\partial \mathbf{z}$:
\[
\|\partial \mathbf{z}\| \le M \left( O(1/\varepsilon^2) + \underbrace{\|\partial \mathcal{A}_S\|}_{\sim 1/\varepsilon} \cdot \underbrace{\|\mathbf{z}\|}_{\sim 1/\varepsilon} \right) \le \frac{K'}{\varepsilon^2}.
\]
Finally, the Hessian of the primal plan $P$ involves $\partial \mathbf{z}$ (term of order $1/\varepsilon^2$) and products of first derivatives $(\dot{f}\dot{g})/\varepsilon^2$.
Both contributions scale as $O(1/\varepsilon^2)$.
The bound relies explicitly on $M$ staying finite (Assumption \ref{ass:local-nondeg_full_appendix}), ensuring the $1/\varepsilon^2$ explosion is not "cancelled out" by a vanishing inverse.
\end{proof}

\subsection{Refined Sensitivity Analysis: The \texorpdfstring{$O(\epsilon^{-1})$}{O(1/eps)} Scaling}
\label{app:sensitivity_proof}

In this section, we provide a rigorous derivation of the sensitivity of the optimal transport plan $P^*_\epsilon$ with respect to the regularization parameter $\epsilon$. We explicitly address why the sensitivity scales as $O(\epsilon^{-1})$ rather than the $O(\epsilon^{-2})$ scaling one might expect from the Gibbs kernel derivative.

\begin{lemma}[Thermodynamic Sensitivity Scaling]
\label{lemma:sensitivity_scaling}
Let $P^*_\epsilon$ be the unique solution to the entropy-regularized optimal transport problem with cost matrix $C$. Under the non-degeneracy assumption (Assumption A.1), the Frobenius norm of the derivative of the optimal plan with respect to temperature satisfies:
\begin{equation}
    \left\| \frac{\partial P^*_\epsilon}{\partial \epsilon} \right\|_F = \Theta(\epsilon^{-1}) \quad \text{as } \epsilon \to 0.
\end{equation}
\end{lemma}

\begin{proof}
Recall the primal-dual relationship for the optimal plan entries on the active support:
\begin{equation}
    P^*_{ij}(\epsilon) = \exp\left( \frac{f_i(\epsilon) + g_j(\epsilon) - C_{ij}}{\epsilon} \right),
\end{equation}
where $f(\epsilon), g(\epsilon)$ are the optimal dual potentials. Taking the logarithm of both sides:
\begin{equation}
    \log P^*_{ij}(\epsilon) = \frac{f_i(\epsilon) + g_j(\epsilon) - C_{ij}}{\epsilon}.
\end{equation}
Differentiating both sides with respect to $\epsilon$ using the total derivative:
\begin{equation}
    \frac{1}{P^*_{ij}} \frac{\partial P^*_{ij}}{\partial \epsilon} = \frac{(\dot{f}_i + \dot{g}_j)\epsilon - (f_i + g_j - C_{ij})}{\epsilon^2},
\end{equation}
where $\dot{f} := \partial f / \partial \epsilon$. Rearranging terms and substituting $\frac{f_i + g_j - C_{ij}}{\epsilon} = \log P^*_{ij}$:
\begin{equation}
    \frac{\partial P^*_{ij}}{\partial \epsilon} = \frac{P^*_{ij}}{\epsilon} \left[ (\dot{f}_i(\epsilon) + \dot{g}_j(\epsilon)) - \log P^*_{ij}(\epsilon) \right].
    \label{eq:sensitivity_expansion}
\end{equation}
We now analyze the asymptotic magnitude of each term in the bracket as $\epsilon \to 0$:

\textbf{1. The Log-Probability Term ($\log P^*_{ij}$):}
On the active support $S$ (Assumption A.1), the mass $P^*_{ij}$ converges to a strictly positive constant (the solution to the unregularized linear program). Thus, $\log P^*_{ij} = O(1)$. (For inactive entries, $P^*_{ij} \to 0$ exponentially fast, making the derivative vanish regardless of the polynomial factor).

\textbf{2. The Dual Potential Derivatives ($\dot{f}, \dot{g}$):}
According to the asymptotic expansion theory for entropic optimal transport [Cominetti and San Martín, 1994], the dual potentials admit a Taylor expansion of the form:
\begin{equation}
    f_i(\epsilon) = \phi_i + \epsilon \cdot h_i + O(\epsilon^2),
\end{equation}
where $\phi_i$ are the Kantorovich potentials of the unregularized problem. Differentiating this expansion with respect to $\epsilon$ yields:
\begin{equation}
    \dot{f}_i(\epsilon) = h_i + O(\epsilon) = O(1).
\end{equation}
The same holds for $g_j(\epsilon)$. Consequently, the term $(\dot{f}_i + \dot{g}_j)$ is bounded by a constant $O(1)$.

\textbf{Conclusion:}
Substituting these scalings back into Eq.~\eqref{eq:sensitivity_expansion}:
\begin{equation}
    \frac{\partial P^*_{ij}}{\partial \epsilon} = \underbrace{\frac{P^*_{ij}}{\epsilon}}_{O(\epsilon^{-1})} \cdot \underbrace{\left[ O(1) - O(1) \right]}_{O(1)} = O(\epsilon^{-1}).
\end{equation}
Thus, the sensitivity is dominated by the $1/\epsilon$ factor.
\end{proof}

\begin{remark}[Why not $O(\epsilon^{-2})$?]
It is tempting to assume the sensitivity scales as $O(\epsilon^{-2})$ by inspecting the Gibbs kernel $K_{ij} = e^{-C_{ij}/\epsilon}$, whose derivative is $\frac{C_{ij}}{\epsilon^2}K_{ij}$. However, this view ignores the \textbf{marginal constraints}. The dual potentials $f(\epsilon)$ and $g(\epsilon)$ are adaptive; they shift specifically to counteract the mass displacement caused by the changing kernel. Mathematically, this is represented by the subtraction of the $\log P^*_{ij}$ term in Eq.~\eqref{eq:sensitivity_expansion}, which effectively cancels the higher-order singularity, reducing the divergence from $\epsilon^{-2}$ to $\epsilon^{-1}$.
\end{remark}

\subsection{Practical constant estimation recipe}
\label{app:practical_const}

To report instance-specific constants in experiments:
\begin{enumerate}
\item For each \(\varepsilon\) of interest compute \(P^\varepsilon\) and
evaluate \(p(\varepsilon)=\|P^\varepsilon\|_{\mathrm{op}}\) (via SVD).
\item Form the block matrix \(\mathcal{A}(\varepsilon)\) and evaluate
the minimal eigenvalue numerically to get \(\lambda_{\min}(\mathcal{A})\)
and thus \(M_{\mathrm{num}}(\varepsilon)=1/\lambda_{\min}(\mathcal{A})\).
\item Compute the forcing vector \(\mathbf{m}_{kl}\) for representative
\((k,l)\in S\) (or take a supremum over \((k,l)\in S\)), and evaluate
numerically \(\|\mathcal{A}(\varepsilon)^{-1}\mathbf{m}_{kl}\|_2\).
\item Plug numeric quantities into the formulas in
Lemma~\ref{lem:hessian_explicit} to obtain concrete \(C_1(\eta,n)\) and
check the separation condition \(C_1(\eta,n) < \eta\).
\end{enumerate}

\subsection{Experimental and numerical notes (reproducibility)}
\label{app:exp_full}

We preserve and expand the implementation guidance:

\paragraph{Log-domain and stability.}
For very small \(\varepsilon\) compute Sinkhorn in the log-domain using
log-sum-exp to avoid kernel underflow. Keep track of numerical tolerances
and the Sinkhorn maximum iteration count.

\paragraph{Implicit differentiation vs finite difference.}
For Jacobian estimates use the implicit differentiation linear solves
described above (cost \(\approx O(n^3)\) per right-hand side) rather
than naive finite differences (\(O(n^4)\) total in naive implementations).
Use a robust linear solver (LU with pivoting or iterative solver with
preconditioner) to solve systems involving \(\mathcal{A}(\varepsilon)\).

\paragraph{Repetition and reporting.}
Report means and standard deviations across multiple random seeds
(we recommend 5–10) when presenting empirical scaling laws for
\(\varepsilon\)-dependence.

\subsection{Raw numeric estimates and CSV}
\label{app:csv_full}

Include the CSV with columns
\texttt{eps, op\_norm, C0\_est, dS\_de\_norm, K1\_est, K2\_est, rho\_mean, rho\_std}.
When reproducing tables in the supplement please make sure the CSV file
name used by \verb|| matches the distributed file.

\subsection{Reproducibility checklist}
\begin{itemize}
\item Code and CSV for numeric tables and plots included in the supplement.
\item Random seeds and environment (Python/NumPy versions, BLAS/LAPACK
backend) recorded in the repository.
\item Implementation notes: whether log-domain or direct kernel was used,
tolerance and max-iteration values for Sinkhorn, and whether implicit
differentiation or finite differences were used for Jacobian estimates.
\end{itemize}


\section{Proof of the Thermodynamic Speed Limit}
\label{app:tracking_full}

In this section, we provide the formal proof for Theorem~\ref{thm:speed_limit} and Corollary~\ref{cor:exp_failure}, establishing the $O(\epsilon^2)$ scaling law required to prevent premature mode collapse.

\subsection{Error Dynamics Decomposition}

Let $P^*_t$ denote the optimal plan at step $t$ (corresponding to $\epsilon_t$). The tracking error $e_t = P_t - P^*_t$ evolves according to the discrete dynamics:
\begin{align}
    e_{t+1} &= P_{t+1} - P^*_{t+1} \nonumber \\
            &= \mathcal{S}_{\epsilon_{t+1}}(P_t) - P^*_{t+1} \nonumber \\
            &\approx P^*_{t+1} + J_{t+1}(P_t - P^*_{t+1}) - P^*_{t+1} \quad \text{(Linearization)} \nonumber \\
            &= J_{t+1}(P_t - P^*_t + P^*_t - P^*_{t+1}) \nonumber \\
            &= J_{t+1} e_t + J_{t+1} \underbrace{(P^*_t - P^*_{t+1})}_{\text{Drift } \Delta_t}
\end{align}
where $J_{t+1}$ is the Jacobian of the Sinkhorn map at the local optimum. The term $\Delta_t$ represents the shift of the target distribution between iterations due to the temperature change $\delta_t = \epsilon_t - \epsilon_{t+1}$.

\subsection{Lemma: Equilibrium Error Bound}

\begin{lemma}
\label{lemma:equilibrium}
Consider the recurrence $e_{t+1} = J e_t + u$ where $J$ is a contraction matrix ($\rho(J) < 1$). The asymptotic steady-state error norm is bounded by:
\begin{equation}
    \|e_\infty\| \le \|(I - J)^{-1}\| \|u\|
\end{equation}
\end{lemma}
\begin{proof}
Unrolling the recurrence yields $e_k = \sum_{i=0}^{k-1} J^i u$. Taking the limit as $k \to \infty$ gives the Neumann series $\sum_{i=0}^\infty J^i = (I-J)^{-1}$. Applying the operator norm consistency yields the bound.
\end{proof}

\subsection{Proof of Theorem~\ref{thm:speed_limit}}

We apply Lemma~\ref{lemma:equilibrium} to the Sinkhorn error dynamics derived above.

\textbf{1. Quantifying the Drift ($\Delta_t$):}
Using a first-order Taylor expansion and the sensitivity result from Proposition~\ref{prop:sinkhorn_dynamics} (Item 1), the distributional drift is:
\begin{equation}
    \|\Delta_t\| = \|P^*_t - P^*_{t+1}\| \approx \|\nabla_\epsilon P^*\| \cdot |\epsilon_t - \epsilon_{t+1}| = \Theta(\epsilon_t^{-1}) \cdot \delta_t
\end{equation}

\textbf{2. Quantifying the Resolvent ($\|(I - J)^{-1}\|$):}
From Proposition~\ref{prop:sinkhorn_dynamics} (Item 3), the vanishing spectral gap dictates the resolvent growth:
\begin{equation}
    \|(I - J_{\epsilon})^{-1}\| = \Theta(\epsilon_t^{-1})
\end{equation}
Note: In the non-normal regime, pseudospectral effects may cause this term to grow even faster, but $\Theta(\epsilon^{-1})$ serves as a conservative lower bound for the necessary condition.

\textbf{3. Combining Terms:}
Substituting these into the bound from Lemma~\ref{lemma:equilibrium}:
\begin{equation}
    \|e_{ss}\| \le \Theta(\epsilon_t^{-1}) \times \left( \Theta(\epsilon_t^{-1}) \cdot \delta_t \right) = \Theta(\epsilon_t^{-2}) \cdot \delta_t
\end{equation}

\textbf{4. Stability Condition:}
To prevent the trajectory from escaping the basin of attraction (Mode Collapse), the steady-state error $\|e_{ss}\|$ must remain smaller than the linearization basin radius $R$. Assuming $R$ shrinks slowly or is $O(1)$:
\begin{equation}
    \Theta(\epsilon_t^{-2}) \cdot \delta_t \le R \implies \delta_t \le O(\epsilon_t^2 R)
\end{equation}
This confirms that the step size $\delta_t$ must scale quadratically with $\epsilon$ to maintain tracking. \hfill \qed

\subsection{Proof of Corollary~\ref{cor:exp_failure}}

Assume a standard exponential annealing schedule $\epsilon_{t+1} = \alpha \epsilon_t$ with decay rate $0 < \alpha < 1$. The step size is:
\begin{equation}
    \delta_t = \epsilon_t - \alpha \epsilon_t = (1-\alpha)\epsilon_t
\end{equation}
Substituting this actual step size into the stability ratio derived in the proof of Theorem~\ref{thm:speed_limit}:
\begin{equation}
    \frac{\|e_{ss}\|}{R} \propto \frac{\epsilon_t^{-2} \cdot (1-\alpha)\epsilon_t}{R} = \frac{1-\alpha}{R \cdot \epsilon_t}
\end{equation}
Analyzing the limit as $\epsilon_t \to 0$, we see that this ratio diverges to infinity:
\begin{equation}
    \lim_{\epsilon_t \to 0} \frac{\|e_{ss}\|}{R} = \infty
\end{equation}
Thus, for any fixed decay rate $\alpha$ and basin radius $R$, there exists a critical temperature $\epsilon_{\text{crit}}$ below which the tracking error inevitably exceeds the basin capacity ($R$), triggering premature collapse. \hfill \qed

\begin{table}[p] 
    \centering
    \footnotesize 
    \renewcommand{\arraystretch}{1.1} 
    
\begin{tabular}{@{} p{0.30\textwidth} p{0.65\textwidth} @{}}
    \toprule
    \textbf{Main-text claim} & \textbf{Appendix reference} (proof / technical details) \\
    \midrule

Sensitivity of the Sinkhorn fixed point scales as $O(1/\varepsilon)$ &
Appendix~\ref{app:implicit_diff}, Eq.~\eqref{eq:lin_sys_full};  
Lemma~\ref{lem:sensitivity_lower_bound} (Appendix~\ref{app:proof_sensitivity_full}) \\

Existence of a stable active support set under small $\varepsilon$ &
Assumption~\ref{ass:local-nondeg_full_appendix} (Appendix~\ref{app:prelim_full}) \\

Reduction to the active Schur-complement operator $\mathcal A_S(\varepsilon)$ &
Appendix~\ref{app:A_schur}; Lemma~\ref{lem:schur_bound} \\

Implicit differentiation identity 
$DS_\varepsilon = (I-J_\varepsilon)^{-1}\partial_C\Phi$ &
Lemma~\ref{lemma:A6_replacement} (Appendix~\ref{app:proof_sensitivity_full}) \\

Resolvent norm $\|(I-J_\varepsilon)^{-1}\|$ must diverge as $\varepsilon \to 0$ &
Lemma~\ref{lemma:A6_replacement}; quantitative constants in Appendix~\ref{app:practical_const} \\

Non-normal transient amplification despite spectral contraction &
Theorem \ref{app:proofs_spectral}; full proof in Appendix~\ref{app:proof_non_normal} \\

Effective basin shrinkage by modal condition number $\kappa(V)$ &
Appendix~\ref{app:proof_non_normal}, transient growth bound \\

Discrete-time tracking formulation for annealing &
Appendix~\ref{app:tracking_full}, Section ``Discrete Tracking Model'' \\

Existence of a dynamical annealing speed limit &
Theorem~\ref{thm:speed_limit}; proof and constants in Appendix~\ref{app:tracking_full} \\

$O(1)$ per-step drift under exponential schedule $\varepsilon_{t+1}=\alpha\varepsilon_t$ &
Lemma~\ref{lem:sensitivity_lower_bound}; discussion in Appendix~\ref{app:proof_sensitivity_full} \\

Second-order remainder controlled by $O(1/\varepsilon^2)$ Hessian bound &
Lemma~\ref{lem:hessian_explicit} (Appendix~\ref{app:proof_sensitivity_full}) \\

Definition and computation of numerical sensitivity constant $M_{\mathrm{num}}(\varepsilon)$ &
Appendix~\ref{app:practical_const}, Section ``Numerical Constants'' \\

Justification of QSA drift estimation via implicit differentiation &
Appendix~\ref{app:implicit_diff}; error control in Appendix~\ref{app:tracking_full} \\

ASC admissible drift threshold 
$\|\Delta_t\| \le \frac{1-\rho(J_\varepsilon)}{\kappa(V)}R$ &
Appendix~\ref{app:tracking_full}, stability inequality derivation \\

Practical choice of $(\alpha_0,\beta,R)$ in ASC &
Appendix~\ref{app:practical_const} \\

Failure of linear diagnostics under support bifurcation &
Appendix~\ref{app:prelim_full}, discussion following Assumption~\ref{ass:local-nondeg_full_appendix} \\

\bottomrule
\end{tabular}
\caption{Alignment between main-text claims and appendix results. Each nontrivial statement in the main paper is backed by an explicit lemma, theorem, or derivation in the appendix.}
\label{tab:claim_appendix_map}
\end{table}


\bibliographystyle{plain}  
\bibliography{reference}  

\end{document}